\documentclass[review]{elsarticle}

\usepackage{hyperref}

\usepackage{multirow} 
\usepackage{array}
\usepackage{graphicx}
\usepackage{tikz}
\usepackage{callouts}
\usepackage{bbm}
\usepackage{caption,subcaption}
\usepackage{mathtools}
\usepackage{amsthm}
\usepackage[textfont=small,labelfont=small]{caption}
\usepackage{wrapfig}
\usepackage{hyperref}
\usepackage{amsfonts}

\newcommand{\new}[1]{{\color{black}#1}}

\newtheorem{definition}{Definition}
\newtheorem{proposition}{Proposition}

\usepackage[linesnumbered,algoruled,boxed,lined]{algorithm2e}
\DeclarePairedDelimiter\ceil{\lceil}{\rceil}

\usepackage{amsmath}

\def\vmu{{\boldsymbol{\mu}}}

\def\vh{{\boldsymbol{h}}}

\def\vu{{\boldsymbol{u}}}
\def\vv{{\boldsymbol{v}}}

\def\vy{{\boldsymbol{y}}}
\def\vz{{\boldsymbol{z}}}

\def\mA{{\boldsymbol{A}}}

\def\mF{{\boldsymbol{F}}}

\def\mH{{\boldsymbol{H}}}
\def\mI{{\boldsymbol{I}}}

\def\mP{{\boldsymbol{P}}}

\def\mS{{\boldsymbol{S}}}

\def\mU{{\boldsymbol{U}}}
\def\mV{{\boldsymbol{V}}}
\def\mW{{\boldsymbol{W}}}
\def\mX{{\boldsymbol{X}}}

\def\mZ{{\boldsymbol{Z}}}

\def\gG{{\mathcal{G}}}

\def\gM{{\mathcal{M}}}
\def\gN{{\mathcal{N}}}

\def\gU{{\mathcal{U}}}

\def\sG{{\mathbb{G}}}
\def\sH{{\mathbb{H}}}

\newcommand{\E}{\mathbb{E}}
\newcommand{\Ls}{\mathcal{L}}
\newcommand{\R}{\mathbb{R}}

\journal{Neural Networks}

\begin{document}

\begin{frontmatter}

\title{Embedding Graphs on Grassmann Manifold}

\author[mymainaddress,mysecondaryaddress]{Bingxin Zhou\fnref{equal}\corref{mycorrespondingauthor}}
\cortext[mycorrespondingauthor]{Corresponding author}
\ead{bzho3923@uni.sydney.edu.au}

\author[mymainaddress]{Xuebin Zheng\fnref{equal}}
\fntext[equal]{Both authors contribute equally}

\author[mysecondaryaddress,mythirdaddress]{Yu Guang Wang}
\author[myfourthaddress]{Ming Li}
\author[mymainaddress]{Junbin Gao}

\address[mymainaddress]{The University of Sydney Business School, The University of Sydney, NSW, Australia}
\address[mysecondaryaddress]{Institute of Natural Sciences and School of Mathematical Sciences, Shanghai Jiao Tong University, Shanghai, China}
\address[mythirdaddress]{School of Mathematics and Statistics, The University of New South Wales, Australia.}
\address[myfourthaddress]{Key Laboratory of Intelligent Education Technology and Application of Zhejiang Province, Zhejiang Normal University, China}

\begin{abstract}
Learning efficient graph representation is the key to \new{favorably} addressing downstream tasks on graphs, such as node or graph property prediction. Given the non-Euclidean structural property of graphs, preserving the original graph data's similarity relationship in the embedded space needs specific tools and a similarity metric. This paper develops \new{a new graph representation learning scheme, namely \textsc{Egg}, which embeds approximated second-order graph characteristics into a Grassmann manifold}. The proposed \new{strategy leverages graph convolutions to learn hidden representations of the corresponding subspace of the graph, which is then mapped} to a Grassmann point of a low dimensional manifold through \new{truncated} singular value decomposition (SVD). \new{The established graph embedding approximates denoised correlationship of node attributes, as implemented in the form of a symmetric matrix space for Euclidean calculation. The effectiveness of \textsc{Egg} is demonstrated using both clustering and classification tasks at the node level and graph level. It outperforms baseline models on various benchmarks.}
\end{abstract}


\end{frontmatter}

\section{Introduction}
\label{sec:intro}
\new{Graph neural networks (GNNs), as one of the most prominent avenues in geometric deep learning, have received growing attention over the last few years \cite{bronstein2017geometric,hamilton2020graph,wu2020comprehensive,zhou2020graph,zhang2020deep}. Common to many GNN-based predictive tasks,} distilling key features and structural information from the given graph data \new{stays within the core of designing an effective graph representation learning.}

Graph convolution \cite{BrZaSzLe2013}, \new{especially graph neural message passing \cite{gilmer2017neural},} provides an efficient \new{expression to} the information flow \new{of the underlying graph through aggregating} regional features over the node neighborhood. For a \new{set of multiple graphs of} varying size and \new{topological structure, employing arbitrary graph convolutions fails to coincide the size of graph representation. Instead, an appropriate graph pooling scheme is required to establish graph-level representations of a uniform scale. Furthermore, a handful of pooling strategies has been proposed to scale down the graph embedding by extracting the key components of the graph representation. Depending on whether the hidden attributes are coarsened along the adjacency matrix, different strategies are categorized as either global \cite{wang2020second,xu2018how,zheng2021framelets} or hierarchical pooling \cite{gao2019graph,zheng2020mathnet,ma2019eigen}.}

\new{While each operation designs its unique standard for graph coarsening and feature extraction, common to all pooling strategies is the requirement} of node permutation invariance. \new{When employing feature extraction,} the node order of an undirected graph must not incline the network to \new{excessively concentrate on individual} attributes. \new{T}he permutation invariance \new{property} is intuitive in heuristic pooling operations like summation, averaging, and maximization, where the aggregation rules of \new{regional patterns} do not rely on the \new{arrangement} of nodes \new{sequence. One possibility to define a permutation invariant pooling operation on graph topology is to view nodes of a graph as} a set of elements. \new{For instance, the authors of  \cite{KolouriNaderializadehRohdeHoffmann2021} establish} the Wasserstein embedding of a node set \new{under} the linear optimal transport framework \cite{Memoli2011}\new{. A} proper Wasserstein metric \new{is designed to match pairs of node sets while avoiding a} direct node selection criterion, as \new{the hierarchical} pooling schemes. 

\begin{figure}
    \centering
    \includegraphics[width=\textwidth]{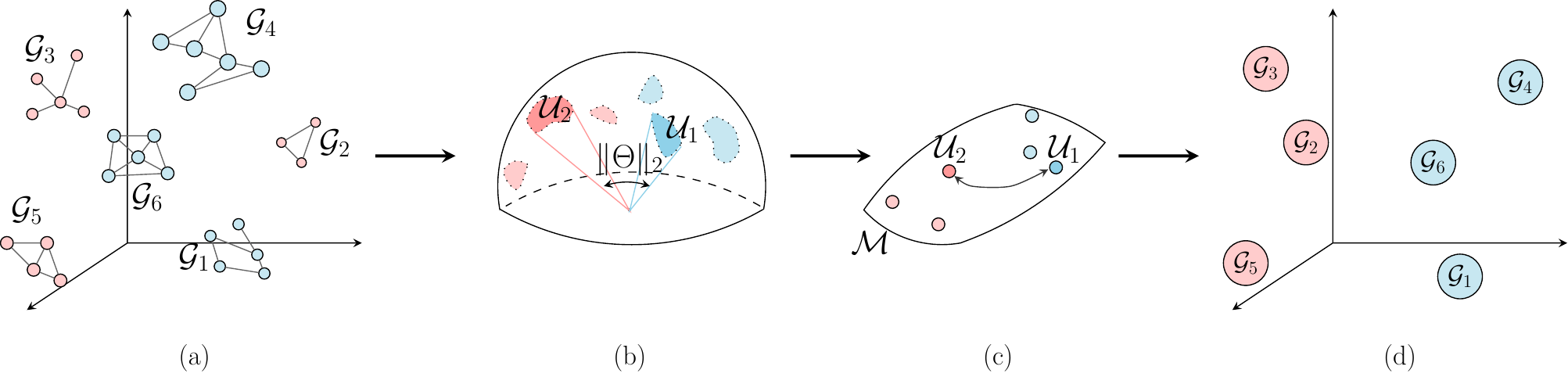}
    \caption{Computational principle of the proposed framework \textsc{Egg}. \new{(a) Given graph $\gG_i$, i.e., a set of nodes, the target is to train a model that assigns a label to each of them. (b)
    Every $\gG_i$ rectifies a subspace of its most representative bit on the manifold space, where their geodesic distance is measured by their principal angle. (c) From the perspective of orthonormal basis, these subspaces can be embedded to the Grassmann points of a Grassmann manifold, where similar points have a small distance. (d) These Grassmann points support an easy projection operation to the space of symmetric matrices for deep learning tasks, such as classification and clustering.}}
    \label{fig:architecture}
\end{figure}

\new{T}his work \new{instead studies the expression of the node set of a graph through} \textit{manifold learning} \new{\cite{AbsilMahonySepulchre2008}}. \new{Recall that the fundamental assumption of a smooth graph in the design of graph convolutions is that spatially connected nodes are likely to share similar characteristics and the nodes from the same class are likely to have similar attributes. Naturally, the hidden representation of an attributed graph creates a subspace of lower dimension, or equivalently, a} point \new{of} a Grassmann manifold. \new{Consequently, a sophisticated} learning task over graphs of \new{varying} size and \new{topology accomplishes its transformation} to a \new{new} learning task \new{over} Grassmann points \new{of} a \new{fixed-dimensional} Grassmann manifold.

We \new{name} th\new{e above} embedding strategy of mapping graph\new{s to} Grassmann \new{points} as \textbf{E}mbedding \textbf{G}raphs on a \textbf{G}rassmann manifold, or \textsc{Egg} for short. \new{This framework is} a generic method to express graphs with a Grassmann manifold subspace analysis\new{. As demonstrated in} Figure~\ref{fig:architecture}\new{, each representation of node sets establishes a subspace of indeterminate dimensions. \textsc{EGG} then embeds these subspaces to Grassmann points of a Grassmannian $\gM$, where every point is explicitly represented by an orthonormal matrix. Furthermore, these Grassmann points support an effortless inversion} to the Euclidean space \new{through symmetric projection}, where \new{new representations} are of the same \new{dimension}, and they \new{are ready for} conventional graph classification \new{or segmentation tasks}. 

\new{In comparison to existing graph distilling strategies, this new} architecture design is one practice of exploiting principal components of graphs from non-linear information transformation. \new{Each node community formulates a subspace of coarsened and smoother higher-level expression. While conventional graph aggregation generally requires stacked convolution or pooling layers to allow non-linear propagation, \textsc{Egg} leverages the truncated singular value decomposition (trSVD) to directly compress the principal components and construct the smooth subspace. In addition, the projected results from these} acquired graph Grassmann embeddings \new{approximates the second-order covariance of node attributes, which gains more expressive power than typical first-order expression from the conventional aggregation and distilling operations. Moreover, the proposed \textsc{Egg} for graph-level tasks guarantees the critical property of permutation invariance, which is an essential requirement of a qualified graph pooling design, yet it has been ignored by many existing methods}. 

The preliminary idea \new{of the developed framework \textsc{Egg}} was \new{first introduced in a workshop paper} \cite{zhou2021grassmann}. \new{This extension provides abundant details for understanding the rationale} and paradigm of the proposed graph embedding scheme. \new{Furthermore, the embedding strategy is expanded from graph pooling applications to more general scenarios, where in this complete work we} exploit \new{the possibility of handling lower-level} unsupervised learning tasks of node \new{segmentation. Additional investigations are addressed to interpret the learned expression and avoid the black-box model design.}

The rest of this paper is organized as follows. \new{Section~\ref{sec:review} reviews the previous literature on graph representation learning and Grassmann deep learning applications.} Section~\ref{sec:preliminary} introduces the \new{Grassmann geometry that is closely related}. Section~\ref{sec:principle} details the two critical ingredients of \new{analyzing subspaces in a Grassmannian}. We then demonstrate our methods with two specific applications: graph classification and node clustering. \new{The problems are formulated} in Section~\ref{sec:applications}\new{, and the empirical performances are reported in} Section~\ref{sec:exp}. \new{Further investigations on} the significance of \textsc{Egg} \new{are addressed in} Section~\ref{sec:exp++}.

\new{
\section{Related Work}
\label{sec:review}
An attributed} undirected graph is denoted as $\gG_i = (V_i, E_i, \mX_i)$ of $n_i:=|V_i|$ nodes and $|E_i|$ edges. The node is featured by $\mX_i\in\R^{n_i\times d}$, and the (weighted) edges 
for the structure information are described by an adjacency matrix $\mA\in\R^{n_i\times n_i}$. \new{Since the graph topology provides additional information, \emph{graph representation learning} aims at encoding such a structural expression to conventional vector representations for deep learning models that assign labels to instances.} A node-level graph learning task assigns a label $\{\vy_i\}$ to each node of the graph $\gG_i$, while a graph-level task finds a sequence of $N$ labels $\{\vy_1, \vy_2, \cdots, \vy_N\}$ from a set of input graphs $\sG=\{\gG_1,\dots,\gG_N\}$. Depending on the nature of the assigned labels, the learning task can be categorized to either regression or classification.

\new{
\paragraph{Spatial Graph Convolution}
The emerging development of graph neural networks (GNNs) generates enormous work for graph representation learning. Typically, the topological embedding is realized by graph convolutional layers. A spatial-based propagation rule \cite{gilmer2017neural} leverages proper feature extraction and aggregation from the central node's local community or the neighborhood. The propagation rule can be designed as flexible as weighted average \cite{kipf2016semi,he2020lightgcn,bo2021beyond}, concatenation \cite{hamilton2017inductive,xu2018representation}, learnable attention \cite{velivckovic2017graph,wang2019heterogeneous,kim2021how}, or other adaptive choices \cite{brockschmidt2020gnn,tailor2021we}. To allow a broader receptive field, multiple convolution layers are frequently stacked for multi-hop neighborhood aggregation. Nevertheless, the majority of the aggregation rules are carried out in the first-order space, which omits the second-order covariance information that can capture insightful non-linear relationships of the feature attributes \cite{lin2015bilinear,wang2017spatiotemporal}.

\paragraph{Graph Pooling and Down-sampling}
As a graph-level learning task involves multiple graphs of a diverse number of nodes, it is crucial for GNNs to unify the dimension of the output graph representation by pooling operations. For instance, \textsc{TopKPool} \cite{cangea2018towards} formulates a score function to rank graph nodes and picks the parent nodes of subgraphs from graph clusters to hierarchically coarsen a graph. Other research enhances the selection efficacy through a carefully designed scoring mechanism, such as multilayer perceptron (MLP) \cite{lee2019self} or attention \cite{li2016gated,knyazev2019understanding}. The authors of \cite{gao2019graph,zheng2020mathnet,ma2019eigen} extended graph coarsening conventions to different slicing principles, although the role of graph clustering and its influence on local pooling has been challenged by the literature \cite{wang2020second,mesquita2020rethinking}. Alternatively, global graph pooling strategies are pumped out in practice with a simpler design and comparable performance \cite{xu2018how,zheng2021framelets}.

\paragraph{Grassmann Manifold in Deep Learning}
Grassmann manifolds play an important role in recommender systems \cite{dai2012geometric,boumal2015low}, computer vision \cite{lui2012advances,minh2016algorithmic} and pattern recognition \cite{huang2015projection,slama2015accurate}. Grassmann learning exploits subspace-invariant features and harnesses the structural information of sample sets, which improves the prediction performance of a model with lower complexity and higher robustness. Due to these privileges, the Grassmann manifold is often approached as a tool of nonlinear dimensionality reduction \cite{koch2007dynamical,ngo2012scaled,dong2014clustering,zhou2019manifold} or optimization objectives \cite{edelman1998geometry,absil2004riemannian}. While direct computations on a Grassmannian can be sophisticated, other research investigates pipeline Grassmann points to} a Grassmann learning algorithm, such as Deep Grassmann Networks \cite{HuangWuGool2016} and Grassmann clustering \cite{WangHuGaoSunYin2014,WangHuGaoSunChenAliYin2017}.

\section{\new{Grassmann Geometry}}
\label{sec:preliminary}
This section overviews the mathematical formulation of the Grassmann points and Grassmann manifold. \new{We also discuss the measurement on} the geodesic distance \new{for Grassmann points, as well as their} Euclidean counterparts.

\subsection{Grassmann Manifold}
Grassmann manifold is a manifold of matrices of a specific rank. Each \new{Grassmann} point \new{is represented by an orthonormal matrix, and} it corresponds to a subspace of the underlying real Euclidean space. To say it precisely,

\begin{definition}[Grassmann Manifold \cite{AbsilMahonySepulchre2008}]
    \new{The} Grassmann manifold \new{$\gM(p,m)$} ($p\leq m$) consists of all $p$-dimensional subspaces \new{of the Euclidean space} $\R^m$\new{, i.e.,}
    $$\gM(p,m)=\{\gU\subset\R^{m}: \gU\text{ is a subspace, } \operatorname{dim}(\gU)=p\}.$$
\end{definition}

A \new{particular} Grassmann point $\gU\in\gM(p,m)$ \new{is identified by an orthonormal matrix $\mU\in\R^{m\times p}$, which is an equivalence class of all rank-$p$ matrices that spans $\gU$. That is,} $\gM(p,m)=\bigl\{\mathrm{span}(\mU):\mU\in\R^{m\times p}, \mU^{\top}\mU=\mathbb{I}_p\bigr\}$. \new{Furthermore, this subspace is identified uniquely by a projector $\Pi(\mU)$ on $\gU$, such that $\Pi(\mU) =\mU \mU^{\top}$.}
The Grassmann manifold is an abstract quotient manifold that one can represent in many ways, such as the Lie group theory \cite{bendokat2020grassmann,gallivan2003efficient}. \new{To best allow convenient algebraic calculation, t}his work constructs Grassmann points from the perspective of projection matrices \cite{chikuse2003statistics}.

\subsection{Subspace Distance}
The distance between \new{two} Grassmann points is measured differently \new{from} the conventional Euclidean metric \new{due to the curvature of the Grassmannian. The} \emph{geodesic distance} \new{is thus defined as} the length of the shortest path along the manifold between two points\new{, which is a function of the principal angles of the two subspaces or analogously the two Grassmann points, as we introduce now}. 

\begin{definition}[Principal angle]
    Given two Grassmann points $\gU_1,\gU_2\in\gM(p,m)$ and their orthonormal bases $\mU_1=[(\vu_1)_1, \cdots, (\vu_1)_p]$, $\mU_2=[(\vu_2)_1,\cdots,(\vu_2)_p]\in \R^{m\times p}$, we define their principal angles $0\leq\theta_1\leq\cdots\leq\theta_p\leq\frac{\pi}{2}$ recursively by 
    \begin{align*}
    &cos(\theta_i)=\max (\vu_1)_i^{\top}(\vu_2)_i \\
    \text{s.t.}\;\; &\|\vu_1\|_2=\|\vu_2\|_2=1,\quad (\vu_1)_i^{\top}(\vu_1)_j=(\vu_2)_i^{\top}(\vu_2)_j=0 \quad \forall j<i.        
    \end{align*}
\end{definition}
The principal angles describe the smallest $p$ angles between all possible bases of the two $p$-dimensional subspaces ($\mU_1$ and $\mU_2$). With a sequence of principal angles $\Theta=[\theta_1,\cdots,\theta_p]$, the geodesic distance between the two Grassmann points is a function of the principal angles, i.e., $d(\mU_1,\mU_2)=\|\Theta\|_2$.

In literature, there exist many other measurements to describe the discrepancy between subspaces, so that a closed-form solution for optimization on Grassmann manifold becomes possible. For example, the \emph{projection distance} \cite{chikuse2003statistics} embeds a Grassmann manifold $\gM(p,m)$ into a higher $m$ dimensional Euclidean space in the form of Symmetric Positive-Definite (SPD) matrix; the chordal distance and the Procrustes distance \cite{ye2016schubert} measures the total squared sine angle between Grassmann points, which is usually used for shape analysis. This work \new{follows} the first measure of the projection distance, which also supports kernelized Grassmann learning \cite{harandi2014expanding} and has been well-explored in learning low-rank approximation \cite{ngo2012scaled,dong2014clustering,zhou2019manifold} and pattern recognition \cite{huang2015projection,slama2015accurate}.

\section{\new{Grassmannian Subspace Analysis on Graphs}}
\label{sec:principle}
This section provides a guideline for \new{graph smoothing and distilling through node sets embedding to a continuous and smooth Grassmann manifold. As we shall introduce below, our proposed method rectifies graph representations with embedded} structure information to a Grassmann\new{ian} of their feature space \new{at a lower dimension}. The \new{Euclidean representations of these Grassmann instances from the projection perspective can be considered as an approximate version of the} feature \new{correlations that eliminates unnecessary variances. A Grassmann embedding appends non-linear smoothing effects to the graph representations that are usually achieved by stacking up fully-connected layers. The output representation meets the key requirements of a standard graph embedding scheme, and it allows arbitrary computations adapted to a Euclidean space.}

\subsection{\new{Problem Formulation}}
We begin with a set of hidden representations $\sH=\{\mH_1,\dots,\mH_N\}$ from graph convolutional layers for a given set of $N$ graphs $\sG=\{\gG_1,\dots,\gG_N\}$. Here $\mH_i\in\R^{n_i\times m}$ is with respect to $n_i$ nodes in $\gG_i$ and \new{a number of} $m$ hidden \new{neurons for} the last graph convolution \new{operation}. The $\mH_i$ for graph-level learning tasks requires a graph representation that is irrelevant to the node size and has an identical dimension to other $\mH\in\sH$. \new{Therefore, at the rectification step,} each graph representation of $\sH$ to a Grassmann point \new{is identified with an orthonormal basis of $\mH$, and is aligned to} the same Grassmannian. 

\subsection{Manifold Rectification}
Suppose we have obtained a graph hidden representation $\mH\in\sH\subset\R^{n\times m}$ \new{as well as} its row-generated subspace $\text{span}(\mH^{\top})$, which as mentioned can be achieved by \new{employing} one or multiple layers of graph convolution. \new{Our target is} to find a concrete Grassmann representation \new{with} most variations of the data. \new{We characterize the representation as} an orthogonal basis of the subspace, \new{which} can be rectified in several ways, \new{such as the} QR decomposition of matrix $\mH^{\top}$, as demonstrated in \cite{HuangWuGool2016}. Here we consider another classic method of the truncated singular value decomposition (SVD) to find the best low-rank approximation of the hidden feature space $\mH^{\top}$ in the sense of the least-squares \cite{chikuse2003statistics}. 

\new{The} preliminary goal of \new{employing} the manifold embedding is to \new{establish a graph representation of a} unite dimension $k$. \new{We therefore leverage the truncated SVD on $\mH$ to obtain} the most representative basis $\mU=[\vu_1, \vu_2,\dots,\vu_p]$ \new{of the subspace $\text{span}(\mH^{\top})$, i.e.,}
\begin{equation}\label{eqn:svd}
    \mH^{\top} =\mU \mS \mV^{\top},
\end{equation}
\new{where the Grassmann point $\gU = [\mU]$ is an equivalence class of $\mU$}. The $\mU\in\R^{m\times k}$ \new{is an orthonormal basis} with $\text{rank}(\mH^{\top})=k, 1\leq k\leq \min\{m,n\}$. The diagonal $\mS:=\text{Diag}([\sigma_1,\dots,\sigma_k]) \in\R^{k\times k}$ contains $k$ singular values sorted in the descending order, where \new{$\sigma_l$ gives} the percentage importance of \new{$\vu_l$}. The corresponding singular vectors constitute $\mV=[\vv_1, \dots, \vv_k]\in\R^{n\times k}$.

\new{Rather than using the full $\mU\in\R^{m\times k}$, we only preserve the} first $p$-columns of $\mU$ ($p\leq k$)\new{, denoted as $\mU_p$, to include} the most important \new{$p$} components of the original space $\mH$. \new{T}he subspace \new{span($\mU_p$)} composes a Grassmann\new{ian} $\gU:=[\mU_p]$ 
\new{in} $\gM(p,m)$. \new{In practice, it could potentially hurt the expressiveness of graph embedding for defining an identical relatively small subspace dimension $p$ for all graphs, since real-world datasets could have a great number of graphs with a large variation on node sizes from a few to thousands. Instead, we let $p$ for a Grassmann point $[\mU_d]\in\gM(p,m)$ be determined by $p=\sum_{i=1}^k\mathbbm{1}\{\sigma_i > r\}$. The $\sigma_i$ corresponds to unit singular values from \eqref{eqn:svd}, and $r$ denotes the global threshold of the percentage importance. However, all these Grassmann points $\{\gU\}$ can be naturally mapped to the embedded Grassmannian $\gM(p_{\text{max}},m)$, with $p_{\text{max}}$ the highest $p$ over all Grassmann points, so they are still at the same space with accessible geodesic distance.}

\new{The} embedding operation \new{is compatible to arbitrary} hidden representation\new{s of} $\sH$. \new{Intuitively, it} regards a \new{(sub)}graph with a set of $n$ \new{attributed} nodes as a $p$-dimensional subspace. \new{While a variant of size $n$ can be observed from different node sets, the rectified Grassmann points from the underlying embedding operation are projected to the same} Grassmann manifold, where \new{the geodesic distance of two points reflects the similarity of two sets, or graph instances in analogue.} Such similarity provides a criterion for \new{distance-based training tasks, such as} clustering or classification.

The rectification step \new{by the truncated SVD} can be considered as a non-linear transformation of the feature space that extracts the most \new{powerful subspace expression of the node space} and \new{view} it as a Grassmann point. \new{At this stage, node sets of varying size are embedded to a common Grassmann manifold, where each of them is represented by a subspace of orthonormal basis $\mU_p\in \mathbb{R}^{m\times p}$. While it is feasible to compute the geodesic distance of Grassmann points, projecting them back to the Euclidean space is preferred by conventional deep learning modules. We now introduce the projection operation of a Grassmann point to its associated Euclidean representation.}

\subsection{Projection Embedding}
\new{The rectification step establishes} a set of Grassmann points $\{\gU_1,\dots,\gU_N\}\subset\gM(p,m)$ \new{as well as their matrix representation $\{\mU_1,\dots,\mU_N\}$ from the orthonormal basis perspective. One can establish} follow-up learning \new{schemes} on $\{\gU_1,\dots,\gU_N\}$ \new{with the} Grassmann geometry \cite{AbsilMahonySepulchre2008}. \new{Alternatively, the} Grassmann \new{points can first be projected} to the space of symmetric matrices $\text{Sym}(m)$ \new{to allow calculations based on Euclidean space. Define}
\begin{equation}
\Pi: \gM(p,m) \longrightarrow \text{Sym}(m), \quad \Pi(\mU) = \mU \mU^{\top}. 
\label{Eq:Embedding}
\end{equation} 
The projected representation $\text{Sym}(m)$ allows general Euclidean measures \new{to allow} conventional deep learning methods \new{such as fully-connected layers}.

\new{After the rectification and the projection step,} the graph latent representation \new{$\mH$ is transformed} to \new{a} symmetric positive definite (SPD) matrix $\Pi(\mU_p) =\mU_p \mU_p^{\top}$. Th\new{is} SPD matrix representation \new{is} an analog to a bilinear mapping, \new{and it} captures the second-order statistics that better reflect\new{s} regional features \new{of $\mH$} \cite{tuzel2006region}. Moreover, the rectified representation $\mU$ from \eqref{eqn:svd} \new{gains robustness as a result of approximating the} covariance matrix $\mU_p \mU_p^{\top}$. This projected Euclidean representation \new{is feasible} for various tasks. For example, in graph property prediction, \new{a vectorized $\Pi(\mU_p)$ can be employed} as the readout \new{train. We summarize the main steps of the Grassmann embedding in Algorithm~\ref{alg:egg}}.

\begin{algorithm}[t]
    \hsize=\textwidth
    \SetKwData{step}{Step}
    \SetKwInOut{Input}{Input}\SetKwInOut{Output}{Output}
    \BlankLine
    \Input{Hidden representation $\mH$, dimension $p$.}
    \Output{A graph representation.}
    (for graph-level tasks): Transpose $\mH$ to $\mH^{\top}$. \tcp*[f]{Transpose\hspace{6mm}}\\
    Find the $p$-dimensional low-rank representation of $\gU$ by \eqref{eqn:svd}.\\ \tcp*[f]{Manifold Rectification\hspace{6mm}}\\
    Project $\gU$ to an SPD matrix \eqref{Eq:Embedding}.\\ \tcp*[f]{Projection Embedding\hspace{6mm}}\\
    \caption{\textsc{Egg} for graph representation learning}
    \label{alg:egg}
    \end{algorithm}

\subsection{Stable SVD for Backward Propagation}
\new{While we package the embedding operations in an end-to-end learning framework, it is essential to develop a computational strategy for SVD that is reliable in} back-propagation (BP) \new{of} deep neural networks.
\new{The rest of this section} give\new{s} the derivation of BP for the employed truncated SVD\new{, which} is numerically stable especially in the case when the input matrix \new{involves} extremely small singular values. \new{We denote} two orthonormal matrices $\mU\in\R^{m\times k}$, $\mV\in\R^{n\times k}$, and $\mS=\mathrm{diag}(s_1, s_2,\ldots,s_k)\in\R^{k\times k}$ \new{as the output of} SVD on $H^{\top}$ during forward-propagation. \new{To update $H^{\top}$ in BP, its gradient is calculated by}
\begin{align*}
    \nabla_{\mH^{\top}}f=
    &\left[\mU\left(\mF\circ\left[\mU^{\top}\overline{\mU}-\overline{\mU}^{\top}\mU\right]\right)\mS+\left(\mI_m-\mU\mU^{\top}\right)\overline{\mU}\mS^{-1}\right]\mV^{\top}\\
    &+\mU\left[\mS\left(\mF\circ\left[\mV^{\top}\overline{\mV}-\overline{\mV}^{\top}\mV\right]\right)\mV^{\top}+\mS^{-1}\overline{\mV}^{\top}\left(\mI_n-\mV\mV^{\top}\right)\right]\\
    &+\mU\left(\mI_k\circ\overline{\mS}\right)\mV^{\top},
\end{align*}
where $\mF_{ij}=\frac{1}{s_j^2-s_i^2}\cdot\mathbbm{1}\{i\neq j\}$ that satisfies the identity $\mF^{\top}=-\mF$. The calculations of $\mF$ and $\mS^{-1}$ are often numerically unstable due to the possible near-zero singular values. To circumvent this difficulty, we \new{follow the authors of }\cite{huang2017riemannian} \new{and introduce the following trick on $\mS$}:
$$
\mS_{i, i}^{\mathrm{new}}=\mS_{i, i}\cdot\mathbbm{1}\{\mS_{i, i}>\epsilon\}+\epsilon\cdot\mathbbm{1}\{\mS_{i, i}\leq\epsilon\},
$$
where $\epsilon$ is a small number and can usually be set to $10^{-12}$. In practice, we replace $\mS$ by the modified matrix $\mS^{\mathrm{new}}$ in BP \new{to avoid $0$ values in $\mS$}.

\section{Applications on Graph}
\label{sec:applications}
\new{T}his section appl\new{ies} the proposed \textsc{Egg} to two \new{distinct graph learning} tasks: node clustering and graph classification. We \new{start with formulating} the two problems to be solved, \new{following the designed} model structure \new{of} the two addressed problems.

\subsection{Graph Classification}
\new{A graph-level representation learning task, such as graph classification and regression, takes multiple graphs} $\sG=\{\gG_1,\dots,\gG_N\}$ \new{as the input to train a feasible model that} makes correct assignment $\vy_i=g(\mX_i,\mA_i)$. \new{T}he node sizes of different graphs are merely identical, \new{and} it is \new{the duty of} graph pooling to learn from a hidden graph representation $\mH_i=f_1(\mX_i,\mA_i)$ so that \new{the} graph is summarized to $\vh_i=f_2(\mH_i)$ with a determined length. Th\new{e} representation \new{is later employed for} label prediction, i.e., $\vy_i=f_3(\vh_i)$.

\new{With GNNs, t}he first step of $f_1(\cdot)$ is usually executed by graph convolutional layers to extract structure and node features, which outputs a hidden representation $\mH\in\R^{n\times m}$ for an arbitrary graph of $n$ nodes. \new{A graph pooling strategy is then selected to design a proper $f_2(\cdot)$ that unifies the dimension of representations to all the graphs.} The proposed \textsc{Egg} defines $f_2(\cdot)$ by embedding each $\mH$ to a Grassmann point of $\gM(p,m)$. Specifically, \new{it} calculate\new{s} a $k$-dimensional subspace of $\mH^{\top}$ by \new{the} truncated SVD. The $\mU$ from \eqref{eqn:svd} \new{denotes} a Grassmann point \new{associated with} a graph. \new{To allow Euclidean computations,} it can be projected to a symmetric matrix by \eqref{Eq:Embedding} \new{and} be flattened and sent to a classifier, such as fully-connected layers. Th\new{e} embedding step is also illustrated in Fig\new{ure}~\ref{fig:architecture} and Algorithm~\ref{alg:egg}.

The validity of \new{\textsc{Egg}} for graph pooling is guaranteed by the two essential requirements: \emph{uni-size representation} and \emph{permutation invariance}, which we shall check now.

\begin{proposition}
\textsc{Egg} always produces a graph embedding $g\in\R^{\frac{m(m+1)}{2}}$ for the node representation matrix $\mH\in\R^{n\times m}$, regardless of the graph size $n$.
\end{proposition}

\begin{proof}
Given the node representation matrix $\mH\in\R^{n\times m}$ of a graph $\mathcal{G}$ with $n$ nodes and $m$ features, the Grassmann graph embedding gives the output $\mU\mU^{\top}$, where $\mU\in\R^{m\times k}$ with $k=\text{rank}(\mH^{\top})$. Then, the output of \textsc{Egg} for the graph classification is the flattened representation of the upper triangular matrix of $\mU\mU^{\top}\in\R^{m\times m}$\new{. The length of vector representation is thus ${m(m+1)}/{2}$, and it} is independent of the graph size $n$.
\end{proof}

\begin{proposition}
\textsc{Egg} satisfies the requirement of permutation invariance so that it produces the same Grassmann graph embedding under row permutations of the input node representation matrix.
\end{proposition}

\begin{proof}
Suppose $\mH_1$ is the node representation matrix of a graph and let $\mH_2=\mP\mH_1$, where $\mP$ is a permutation matrix. Then,
\begin{equation*}
    \mH_1^{\top} = \mU \mS \mV^{\top},\quad
    \mH_2^{\top} = \mH_1^{\top}\mP^{\top} =\mU\mS\mV^{\top}\mP^{\top}.
\end{equation*}
The Grassmann point for both $\mH_1$ and $\mH_2$ can be accessed by the same matrix $\mU$. Hence, the proposed graph embedding method is permutation invariant.
\end{proof}

\subsection{Node Clustering}
\label{sec:app:clustering}
Node-level tasks make predictions on each node $v$ of a \new{single} graph $\gG$. For node clustering, the target is to segment the full graph to a number of subsets, where nodes from the same subset usually have closer connection to each other, and/or \new{share} similar \new{properties}. 

A typical approach in literature to solve node clustering problems is by \new{using} \emph{Variational Graph Auto-Encoder} (\textsc{VGAE}) \cite{kipf2016variational} \new{to generate a latent representation for a graph and then applying $k$-means \cite{lloyd1982least} for clustering the nodes} , where the \new{\textsc{VGAE}} \new{model pursues the} optimal variational parameters $\mW$ that minimize the variational lower bound
\begin{equation*}
\Ls=\E_{q(\mZ|\mX,\mA)}\bigl[\log p(\mA|\mZ)]-\mathrm{KL}[q(\mZ|\mX,\mA) \| p(\mZ)\bigr].
\end{equation*}
Here $q(\mZ|\mX,\mA)$ is the encoder such that
\begin{equation*}
q(\mZ|\mX,\mA)=\prod_{i=1}^N q(\vz_i|\mX,\mA), \quad
q(\vz_{i}|\mX,\mA)\sim\gN(\vz_i|\vmu_i,\operatorname{diag}(\boldsymbol{\sigma}_i^2)).
\end{equation*}
Both \new{the} mean $\vmu$ and log standard deviation $\log\boldsymbol{\sigma}$ \new{are} approximated by graph convolutional layers, such as \textsc{GCN} \cite{kipf2016semi}. \new{T}he decoder is defined as
\begin{equation*}
p(\mA|\mZ)=\prod_{i=1}^N \prod_{j=1}^N p(A_{i j}|\vz_i, \vz_j),\quad
p(A_{i j}=1|\vz_i, \vz_j)=\sigma(\vz_i^{\top} \vz_j),
\end{equation*}
where $\mA$ is with respect to the adjacency matrix, and $\sigma(\cdot)$ is the activation function. We take the inference set $\mH := [\vmu_1,\cdots,\vmu_m]\in\R^{n\times m}$ as the hidden representation of the graph \new{with} $n$ node\new{s of} hidden size $m$ \new{ and send it} to $\textsc{Egg}$ for Grassmann embedding\new{. T}he consequent $\mU\mU^{\top}\in\R^{n\times n}$ is handled by clustering methods, such as $k$-means \cite{lloyd1982least}, to \new{assign} clusters \new{with $k$ set to be the number of classes of the dataset}. Instead of relying on the pair-wise connection or the feature-space \new{information}, $\textsc{Egg}$ \new{leverages} second-order \new{correlations} of nodes \new{for a proper} segmentation.

\section{Experiment}
\label{sec:exp}
This section evaluates the proposed \new{framework} on graph\new{-level} classification tasks \new{as well as node-level clustering tasks. The former is conducted on six benchmarks of variant} graph sizes, volume, and density\new{, and the latter employs five popular graph datasets of moderate volume. All benchmark datasets and baseline methods are publicly available in the \texttt{PyTorch Geometric} (\texttt{PYG}) \cite{fey2019fast} library. The implementation of \textsc{Egg} is published at \hyperlink{https://github.com/conf20/Egg}{\textcolor{blue}{https://github.com/conf20/Egg}}.
The rest of this section lists the experimental setup and analyzes the performance comparison of the two experiments.}

\subsection{Ablation Study on Graph Pooling}
\label{sec:expPool}

\begin{table}[h]
\caption{Summary of the datasets for graph property prediction tasks.}
\label{table:graph descriptive}
\begin{center}
\resizebox{0.9\linewidth}{!}{
\begin{tabular}{lrrrrrrrr}
\hline
\textbf{Datasets} & \textbf{Proteins} & \textbf{D\&D} & \textbf{NCI1} & \textbf{Mutagen} & \textbf{Collab} & \textbf{Molhiv} \\ \hline
\# graphs & $1,113$ & $1,178$ & $4,110$ & $4,337$ & $5,000$ & $41,127$ \\
\# classes & $2$ & $2$ & $2$ & $2$ & $3$ & $2$ \\
Min \# nodes & $4$ & $30$ & $3$ & $30$ & $32$ & $2$ \\
Max \# Nodes & $620$ & $5,748$ & $111$ & $417$ & $492$ & $222$ \\
Avg \# nodes & $39$ & $284$ & $30$ & $30$ & $74$ & $26$ \\
Avg \# edges & $73$ & $716$ & $32$ & $31$ & $2,458$ & $28$ \\
\# Features  & $3$ & $89$ & $37$ & $14$ & $0$  & $9$ \\ \hline
\end{tabular}
}
\end{center}
\end{table}

\begin{table}[t]
\caption{Performance comparison for graph \new{classification with \textbf{\textsc{GCN} convolution}}. }
\label{table:poolgcn}
\begin{center}
\resizebox{\textwidth}{!}{%
\begin{tabular}{lrrrrrrr} \hline
& \textbf{Proteins} & \textbf{D\&D} & \textbf{NCI1} & \textbf{Mutagen} & \textbf{Collab} & \textbf{Molhiv} \\ \hline 
TopKPool  & $73.48${\scriptsize $\pm3.57$} & $74.87${\scriptsize $\pm4.12$} & $75.11${\scriptsize $\pm3.45$} & $79.84${\scriptsize $\pm2.46$} & $81.18${\scriptsize $\pm0.89$} & $77.11${\scriptsize $\pm1.27$}\\
SAGPool  & $75.89${\scriptsize$\pm2.91$} & $74.96${\scriptsize $\pm3.60$} & $76.30${\scriptsize $\pm1.53$} & $79.86${\scriptsize $\pm2.36$} & $79.26${\scriptsize $\pm5.37$} & $75.36${\scriptsize$\pm1.82$}\\
\new{EDGEPool} & $75.60${\scriptsize $\pm2.40$} & $67.60${\scriptsize $\pm0.51$} & $77.17${\scriptsize $\pm1.49$} & $70.34${\scriptsize $\pm1.69$} & $75.09${\scriptsize $\pm0.81$} & $75.14${\scriptsize $\pm1.66$} \\
\new{PANPool} & $72.41${\scriptsize $\pm3.58$} & $72.52${\scriptsize $\pm4.05$} & $62.82${\scriptsize $\pm3.04$} & $70.14${\scriptsize $\pm1.85$} & $75.78${\scriptsize $\pm2.03$} & $74.18${\scriptsize $\pm1.52$} \\ \hline
SUM & $74.91${\scriptsize $\pm4.08$} & $78.91${\scriptsize$\pm3.37$} & $76.96${\scriptsize$\pm1.70$} & $80.69${\scriptsize$\pm3.26$} & $80.76${\scriptsize $\pm1.56$} & $74.88${\scriptsize$\pm2.64$}\\
AVG & $73.13${\scriptsize $\pm3.18$} & $76.89${\scriptsize $\pm2.23$} & $73.70${\scriptsize $\pm2.55$} & $80.37${\scriptsize $\pm2.44$} & $81.24${\scriptsize $\pm1.34$} & \textcolor{black}{$\boldsymbol{77.69}${\scriptsize $\pm1.17$}} \\
MAX & $73.57${\scriptsize $\pm3.94$} & $75.80${\scriptsize $\pm4.11$} & $75.96${\scriptsize $\pm1.82$} & $78.83${\scriptsize $\pm1.70$} & $82.28${\scriptsize $\pm2.10$} & $76.95${\scriptsize $\pm0.94$}\\ 
Attention & $73.93${\scriptsize $\pm5.37$} & $77.48${\scriptsize $\pm2.65$} & $74.04${\scriptsize $\pm1.27$} & $80.25${\scriptsize $\pm2.22$} & $81.58${\scriptsize $\pm1.72$} & $77.44${\scriptsize $\pm1.27$}\\ \hline
EGG & \textcolor{black}{$\boldsymbol{77.79}${\scriptsize$\pm2.16$}} & \textcolor{black}{$\boldsymbol{79.10}${\scriptsize$\pm2.98$}} & \textcolor{black}{$\boldsymbol{77.80}${\scriptsize$\pm2.01$}} & \textcolor{black}{$\boldsymbol{81.01}${\scriptsize$\pm1.28$}} & \textcolor{black}{$\boldsymbol{82.94}${\scriptsize$\pm1.06$}} & $76.60${\scriptsize $\pm1.10$}\\ \hline 
\end{tabular}
}
\end{center}
\end{table}
\begin{table}[t]
\caption{Performance comparison for graph \new{classification with \textbf{\textsc{GIN} convolution}}. }
\label{table:poolgin}
\begin{center}
\resizebox{\textwidth}{!}{
\begin{tabular}{lrrrrrrr}
\hline
& \textbf{Proteins} & \textbf{D\&D} & \textbf{NCI1} & \textbf{Mutagen} & \textbf{Collab} & \textbf{Molhiv}\\
\hline 
TopKPool & $73.66${\scriptsize $\pm6.00$} & $76.40${\scriptsize $\pm2.32$} & $77.06${\scriptsize $\pm0.90$} & $78.30${\scriptsize $\pm1.36$} & $81.40${\scriptsize $\pm0.94$} & $78.14${\scriptsize $\pm0.62$}\\
SAGPool & $75.95${\scriptsize $\pm4.52$} & $68.94${\scriptsize $\pm7.62$} & $76.97${\scriptsize $\pm2.94$} & $78.86${\scriptsize $\pm1.58$} & $81.76${\scriptsize $\pm1.57$} & $75.26${\scriptsize $\pm2.29$}\\
\new{EDGEPool} & $75.13${\scriptsize $\pm3.62$} & $72.82${\scriptsize $\pm1.40$} & $77.79${\scriptsize $\pm2.80$} & $81.01${\scriptsize $\pm0.82$} & $79.20${\scriptsize $\pm1.66$}& $75.30${\scriptsize $\pm2.01$}\\
\new{PANPool} & $71.43${\scriptsize $\pm2.15$} & $72.75${\scriptsize $\pm2.32$} & $71.68${\scriptsize $\pm4.45$} & $78.09${\scriptsize $\pm1.27$} & $80.22${\scriptsize $\pm2.02$} & $77.18${\scriptsize $\pm1.13$} \\ \hline
SUM & $78.04${\scriptsize $\pm2.30$} & $78.57${\scriptsize $\pm1.26$} & $78.83${\scriptsize $\pm1.49$} & $81.31${\scriptsize $\pm1.10$} & $82.64${\scriptsize $\pm0.85$} & $77.41${\scriptsize $\pm1.16$}\\ 
AVG & $71.70${\scriptsize $\pm2.08$} & $74.37${\scriptsize $\pm1.32$} & $76.55${\scriptsize $\pm1.72$} & $80.97${\scriptsize $\pm1.18$} & \textcolor{black}{$\boldsymbol{83.30}${\scriptsize$\pm0.77$}} & \textcolor{black}{$\boldsymbol{78.21}${\scriptsize$\pm0.90$}} \\
MAX & $76.70${\scriptsize $\pm1.57$} & $77.31${\scriptsize $\pm2.06$} & $79.27${\scriptsize $\pm1.38$} & $80.28${\scriptsize $\pm0.83$} & $80.94${\scriptsize $\pm0.72$} & $78.16${\scriptsize $\pm1.33$}\\ 
Attention & $75.63${\scriptsize $\pm1.13$} & $71.76${\scriptsize $\pm3.26$} & $78.22${\scriptsize $\pm1.32$} & $78.54${\scriptsize $\pm5.37$} & $83.22${\scriptsize $\pm0.30$} & $74.44${\scriptsize $\pm2.12$}\\ \hline
EGG & \textcolor{black}{$\boldsymbol{79.80}${\scriptsize$\pm1.09$}} & \textcolor{black}{$\boldsymbol{81.18}${\scriptsize$\pm1.14$}} & \textcolor{black}{$\boldsymbol{81.31}${\scriptsize$\pm1.55$}} & \textcolor{black}{$\boldsymbol{82.53}${\scriptsize$\pm0.72$}} & $81.32${\scriptsize $\pm0.68$} & $77.82${\scriptsize $\pm0.90$}\\ \hline 
\end{tabular}
}
\end{center}
\end{table}

\subsubsection{Setup}
We benchmark the performance on five binary classification and one multi-class classification tasks. The \textbf{Molhiv} \cite{hu2020open} is from open graph benchmark, and all other datasets are provided by \textbf{TUDataset benchmarks} \cite{morris2020tudataset}. \new{T}he summary statistics of the six benchmark datasets \new{is provided} in Table~\ref{table:graph descriptive}. \new{Most benchmark datasets preserve their original feature attributes except for} the feature-less dataset \textbf{Collab}, \new{where} we follow \cite{xu2018how} to \new{generate new features with} one-hot encoding of node degrees. \new{Also}, virtual nodes \cite{ishiguro2019graph} are included in \textbf{Molhiv} to enhance the learning ability\new{, as is suggested by the authors of \cite{hu2020open}}.

\new{To learn the hidden representation of graph topological embedding for pooling layers, we consider two variants of \textsc{GCN} \cite{kipf2016semi} and \textsc{GIN} \cite{xu2018how}. For the first five datasets from \textbf{TU Dataset}s, we construct three \textsc{GCN} layers for \textbf{NCI1}, and} two \textsc{GCN} layers \new{for the other four datasets to encode graph hidden representation for pooling. Meanwhile, the number of \textsc{GIN} convolutional layers is set to four with the \textsc{JKNet} \cite{xu2018representation} construction. Both models for the largest dataset \textbf{Molhiv}} use four convolutional layers. \new{The learned hidden representation is sent to one of the baselines, following a} two-layer MLP with size $64$ and $16$\new{, respectively}.

\new{A fair comparison of \textsc{Egg} is made against four hierarchical and four global pooling methods. The former includes} \textsc{TopKPool} \cite{cangea2018towards,gao2019graph}, \textsc{SAGPool} \cite{lee2019self} \new{\textsc{EDGEPool} \cite{diehl2019towards,diehl2019edge}, and \textsc{PANPool} \cite{ma2020path}, and the latter chooses} \textsc{Attention} \cite{li2016gated},  \new{\textsc{Summation}, \textsc{Average} and \textsc{Maximization} methods. The official implementation are of all eight baselines are provided by \texttt{PyTorch Geometric} \cite{fey2019fast}}. 

\new{The models are trained on $80\%$ randomly selected samples of the datasets, validated on $10\%$ samples, and tested on the left $10\%$ samples. The main hyper-parameters are fine-tuned with a grid search engine, where we are interested in Learning rate in $\{5\mathrm{e}{-3}, 1\mathrm{e}{-3}, 5\mathrm{e}{-4\}}$, $L_2$ weight decay in $\{5\mathrm{e}{-3}, 1\mathrm{e}{-3}, 5\mathrm{e}{-4}\}$, and hidden units in $\{32, 64\}$ for the convolution layers. For the fully-connected layer, we search in particular the dropout ratio in $\{0, 0.5\}$. For \textsc{Egg}, we include an extra hyper-parameter of the threshold ratio $r$ in truncated SVD, which is set to one of $\{0.5, 0.8\}$. The model stops training whenever} the validation loss stops improving for $20$ consecutive epochs, or the \new{training epoch reaches $200$.}

\subsubsection{Result}
Table~\ref{table:poolgcn} \new{and Table~\ref{table:poolgin}} compare the prediction performance of \textsc{Egg} with \new{the eight} baselines. \new{Follow the convention, w}e report \new{the percentage value of} mean test accuracy for \new{the classification tasks with \textbf{TUDataset}s} and ROC-AUC score for \textbf{Molhiv}. \new{The mean performance score} are averaged over $10$ repetitions \new{with their standard deviations provided after the $\pm$ signs}. \new{In general, our} \textsc{Egg} achieves the top score \new{on all the six tasks} with a lower volatility. The advantage is more salient \new{when the hidden representation of graphs are embedded by \textsc{GIN} convolutions with \textsc{JKNet} structure. \textsc{Egg} constructs a feasible global pooling method that interprets the second-order correlation of the compressed graph expressions, which are more informative than the first-order relationships. As this covariance relationship is captured by a non-linear transformation, i.e., a truncated SVD, one layer of \textsc{Egg} is generally sufficient for graph distilling tasks, and it reliefs the burden of the fine-tuning work in training a deep graph learning model. It should be emphasized that the correlation analysis of \textsc{Egg} relies on the node attributes. When the input graph is feature-less, such as \textbf{Collab}, the performance can be less-promising.}

\begin{table}[t]
\caption{Summary of the datasets for node \new{clustering} tasks.}
\label{tab:stats:node_clustering}
\begin{center}
\resizebox{0.85\textwidth}{!}{
\begin{tabular}{lrrrrr}
\hline
& \textbf{Cora} & \textbf{Citeseer} & \textbf{Pubmed} & \textbf{Wiki-CS} & \textbf{Coauthor-CS}\\ \hline
\# Nodes & $2,708$ & $3,327$ & $19,717$ & $11,701$ & $18,333$\\
\# Edges & $5,429$ & $4,732$ & $44,338$ & $216,123$ & $100,227$ \\
\# Features & $1,433$ & $3,703$ & $500$ & $300$ & $6,805$ \\
\# Clusters & $7$ & $6$ & $3$ & $10$ & $15$ \\
\hline
\end{tabular}
}
\end{center}
\end{table}

\subsection{Node Clustering}
\label{sec:expClustering}
\subsubsection{Setup}
\new{The second experiment validates the design of} \textsc{Egg} to node clustering tasks. \new{Five popular benchmarks, including the three citation networks (\textbf{Cora}, \textbf{Citeseer}, \textbf{Pubmed}) \cite{yang2016revisiting}, \textbf{Wiki-CS} \cite{mernyei2020wiki} and \textbf{Coauthor-CS} \cite{shchur2018pitfalls}, are employed to examine the effectiveness of an additional embedding step of \textsc{Egg}}. The statistics of the \new{five} datasets are summarized in Table~\ref{tab:stats:node_clustering}. In terms of the baseline method, we train a \textsc{VGAE} model \cite{kipf2016variational} to generate the latent representation $\mH\in\R^{n\times m}$ of the graph and then send the $\mH$ matrix to $k$-means \cite{lloyd1982least} for clustering. Based on this baseline architecture, we insert our \textsc{Egg} before the $k$-means as an additional step. Specifically, we first send the learned latent representation $\mH$ to \textsc{Egg}, then perform $k$-means on the output of the \textsc{Egg} procedure.

\new{T}he same structure of the \textsc{VGAE} model \new{is adopted} from \cite{kipf2016variational}, where the encoder is implemented by a two-layer \textsc{GCN} \cite{kipf2016semi} and the decoder is simply given by an inner product between the latent variables as illustrated in Section~\ref{sec:app:clustering}. We train the \textsc{VGAE} model for $200$ epochs using \textsc{Adam} \cite{kingma2014adam} with a learning rate of $0.01$. The dimensions of the hidden layer and the latent space are set to $32$ and $16$ in all experiments, respectively. \new{To obtain the latent representation of nodes, \textsc{VGAE} is trained on a link prediction task of identifying edges and non-edges. Following \cite{kipf2016variational}, we divide the dataset into the training, validation and test sets by a random selection of $85\%$, $5\%$ and $10\%$ of the total node connections. The positive and negative edges share the same amount.} In terms of evaluation, we employ the following four metrics to validate the clustering results: Accuracy (Acc.), Normalized Mutual Information (NMI), Average Rand Index (ARI), and Completeness Score (CS).

\begin{table}[t]
\caption{Performance comparison for node clustering. All scores are averaged over 10 repetitions with the scale \new{between 0 and} 1. The value after $\pm$ is standard deviation.}
\label{table:node_clustering}
\begin{center}
\resizebox{\textwidth}{!}{
\begin{tabular}{llcccc}
\hline
& \textbf{Method} & \textbf{Acc.} & \textbf{NMI} & \textbf{ARI} & \textbf{CS} \\ \hline
\multirow{4}{*}{\rotatebox[origin=c]{90}{Cora}}\hspace{0.1cm} 
& $k$-means & $0.6137${\scriptsize $\pm0.0345$} & $0.4459${\scriptsize $\pm0.0190$} & $0.3782${\scriptsize $\pm0.0312$} & $0.4351${\scriptsize $\pm0.0188$}\\
& $k$-means+EGG{\scriptsize $0.2$} ($53.16\%$) & $0.5207${\scriptsize $\pm0.0553$} & $0.3789${\scriptsize $\pm0.0374$} & $0.2967${\scriptsize $\pm0.0475$} & $0.3702${\scriptsize $\pm0.0360$}\\
& $k$-means+EGG{\scriptsize $0.5$} ($90.73\%$) & $0.6195${\scriptsize $\pm0.0340$} & $0.4345${\scriptsize $\pm0.0269$} & $0.3804${\scriptsize $\pm0.0434$} & $0.4292${\scriptsize $\pm0.0257$}\\
& $k$-means+EGG{\scriptsize $0.8$} ($98.76\%$) & $\mathbf{0.6388}${\scriptsize $\pm0.0386$} & $\mathbf{0.4548}${\scriptsize $\pm0.0158$} & $\mathbf{0.3998}${\scriptsize $\pm0.0293$} & $\mathbf{0.4591}${\scriptsize $\pm0.0191$} \\ \hline
\multirow{4}{*}{\rotatebox[origin=c]{90}{Citeseer}} \hspace{0.1cm}
& $k$-means & $0.4347${\scriptsize $\pm0.0341$} & $0.1931${\scriptsize $\pm0.0343$} & $0.1548${\scriptsize $\pm0.0262$} & $0.1942${\scriptsize $\pm0.0358$}\\
& $k$-means+EGG{\scriptsize $0.2$} ($45.28\%$) & $0.4156${\scriptsize $\pm0.0409$} & $0.1794${\scriptsize $\pm0.0259$} & $0.1527${\scriptsize $\pm0.0294$} & $0.1783${\scriptsize $\pm0.0254$}\\
& $k$-means+EGG{\scriptsize $0.5$} ($88.25\%$) & $0.4683${\scriptsize $\pm0.0323$} & $\mathbf{0.2098}${\scriptsize $\pm0.0283$} & $\mathbf{0.1865}${\scriptsize $\pm0.0301$} & $0.2096${\scriptsize $\pm0.0282$}\\
& $k$-means+EGG{\scriptsize $0.8$} ($97.98\%$) & $\mathbf{0.4698}${\scriptsize $\pm0.0284$} & $0.2071${\scriptsize $\pm0.0235$} & $0.1859${\scriptsize $\pm0.0282$} & $\mathbf{0.2107}${\scriptsize $\pm0.0201$} \\ \hline
\multirow{4}{*}{\rotatebox[origin=c]{90}{Pubmed}} \hspace{0.1cm}
& $k$-means & $0.6469${\scriptsize $\pm0.0232$} & $0.2425${\scriptsize $\pm0.0233$} & $0.2380${\scriptsize $\pm0.0368$} & $0.2393${\scriptsize $\pm0.0233$}\\
& $k$-means+EGG{\scriptsize $0.2$} ($64.98\%$) & $0.6339${\scriptsize $\pm0.0193$} & $0.2311${\scriptsize $\pm0.0162$} & $0.2196${\scriptsize $\pm0.0228$} & $0.2280${\scriptsize $\pm0.0163$}\\
& $k$-means+EGG{\scriptsize $0.5$} ($97.18\%$) & $0.6296${\scriptsize $\pm0.0228$} & $0.2394${\scriptsize $\pm0.0252$} & $0.2171${\scriptsize $\pm0.0302$} & $0.2377${\scriptsize $\pm0.0263$}\\
& $k$-means+EGG{\scriptsize $0.8$} ($99.48\%$) & $\mathbf{0.6521}${\scriptsize $\pm0.0135$} & $\mathbf{0.2532}${\scriptsize $\pm0.0201$} & $\mathbf{0.2436}${\scriptsize $\pm0.0231$} & $\mathbf{0.2521}${\scriptsize $\pm0.0224$} \\ \hline
\multirow{4}{*}{\rotatebox[origin=c]{90}{\new{Wiki-CS}}} \hspace{0.1cm}
& $k$-means & $0.4080${\scriptsize $\pm0.0371$} & $0.3429${\scriptsize $\pm0.0157$} & $0.2156${\scriptsize $\pm0.0240$} & $0.3418${\scriptsize $\pm0.0153$}\\
& $k$-means+EGG{\scriptsize $0.2$} ($52.48\%$) & $0.3875${\scriptsize $\pm0.0392$} & $0.3207${\scriptsize $\pm0.0183$} & $0.2065${\scriptsize $\pm0.0276$} & $0.3195${\scriptsize $\pm0.0179$} \\
& $k$-means+EGG{\scriptsize $0.5$} ($93.21\%$) & $0.4103${\scriptsize $\pm0.0253$}& $0.3319${\scriptsize $\pm0.0132$} & $0.2232${\scriptsize $\pm0.0249$} & $0.3307${\scriptsize $\pm0.0128$} \\
& $k$-means+EGG{\scriptsize $0.8$} ($99.43\%$) & $\mathbf{0.4685}${\scriptsize $\pm0.0210$}& $\mathbf{0.3643}${\scriptsize $\pm0.0193$} & $\mathbf{0.2423}${\scriptsize $\pm0.0223$} & $\mathbf{0.3633}${\scriptsize $\pm0.0181$} \\ \hline
\multirow{4}{*}{\rotatebox[origin=c]{90}{\new{Coauthor-CS}}} \hspace{0.1cm}
& $k$-means & $0.6410${\scriptsize $\pm0.0171$} & $0.6950${\scriptsize $\pm0.0311$} & $0.5442${\scriptsize $\pm0.0427$}& $0.6880${\scriptsize $\pm0.0300$} \\
& $k$-means+EGG{\scriptsize $0.2$} ($45.05\%$) & $0.5483${\scriptsize $\pm0.0250$} & $0.6095${\scriptsize $\pm0.0210$} & $0.4472${\scriptsize $\pm0.0269$} & $0.6050${\scriptsize $\pm0.0235$} \\
& $k$-means+EGG{\scriptsize $0.5$} ($92.67\%$) & $0.5878${\scriptsize $\pm0.0188$} & $0.6750${\scriptsize $\pm0.0150$} & $0.5033${\scriptsize $\pm0.0190$} & $0.6735${\scriptsize $\pm0.0157$}\\
& $k$-means+EGG{\scriptsize $0.8$} ($99.68\%$) & $\mathbf{0.6700}${\scriptsize $\pm0.0201$} & $\mathbf{0.7248}${\scriptsize $\pm0.0195$} & $\mathbf{0.6150}${\scriptsize $\pm0.0205$} & $\mathbf{0.7230}${\scriptsize $\pm0.0193$} \\ \hline
\end{tabular}}
\end{center}
\end{table}

\subsubsection{Result}
Table~\ref{table:node_clustering} \new{reports the experimental results of \textsc{Egg}-enhanced $k$-means with \textsc{VGAE} in node clustering tasks. In addition to the four evaluation metrics, we also attach two additional ratios in the name of $k$-means+EGG$_x$ ($y$), where the first ratio} $x\in(0, 1)$ is used to \new{determine} the number of most important components $p$ in the latent space of $\mH\in\R^{n\times m}$ we have kept, that is $p=\ceil{xm}$. The \new{second} value $y$ indicates the percentage of the information from the latent representation $\mH$ being captured by these $p$ components.

\new{It shows clearly from Table~\ref{table:node_clustering}} that the \new{\textsc{Egg}{\scriptsize 0.8}-enhanced $k$-means} consistently outperforms the plain $k$-means on all the datasets in terms of all four metrics. We \new{also} observe that \textsc{Egg}{\scriptsize 0.5} provide a comparable performance against the baseline, \new{giving} that the first $p=\ceil{0.5m}$ components contains \new{$90\%$ variations} of the \new{graphs'} latent representation. These observations confirm that \textsc{Egg} brings \new{solid} performance gain\new{s to} the node clustering tasks.

\section{Further Investigation}
\label{sec:exp++}
This section validates \textsc{Egg} from three perspectives. \new{T}he sensitivity of the model performance to the newly introduced hyperparameter \new{is tested} in the first part. We then explore the expressiveness of the learned embedding \new{with a t-SNE visualization of the learned hidden representation}. In the last sector, we check the \new{learning behavior of \textsc{Egg} through the} loss curve \new{of training tasks}.

\subsection{Sensitivity to the embedded dimension}
\begin{figure}[t]
    \begin{center}
    \includegraphics[width=0.5\textwidth]{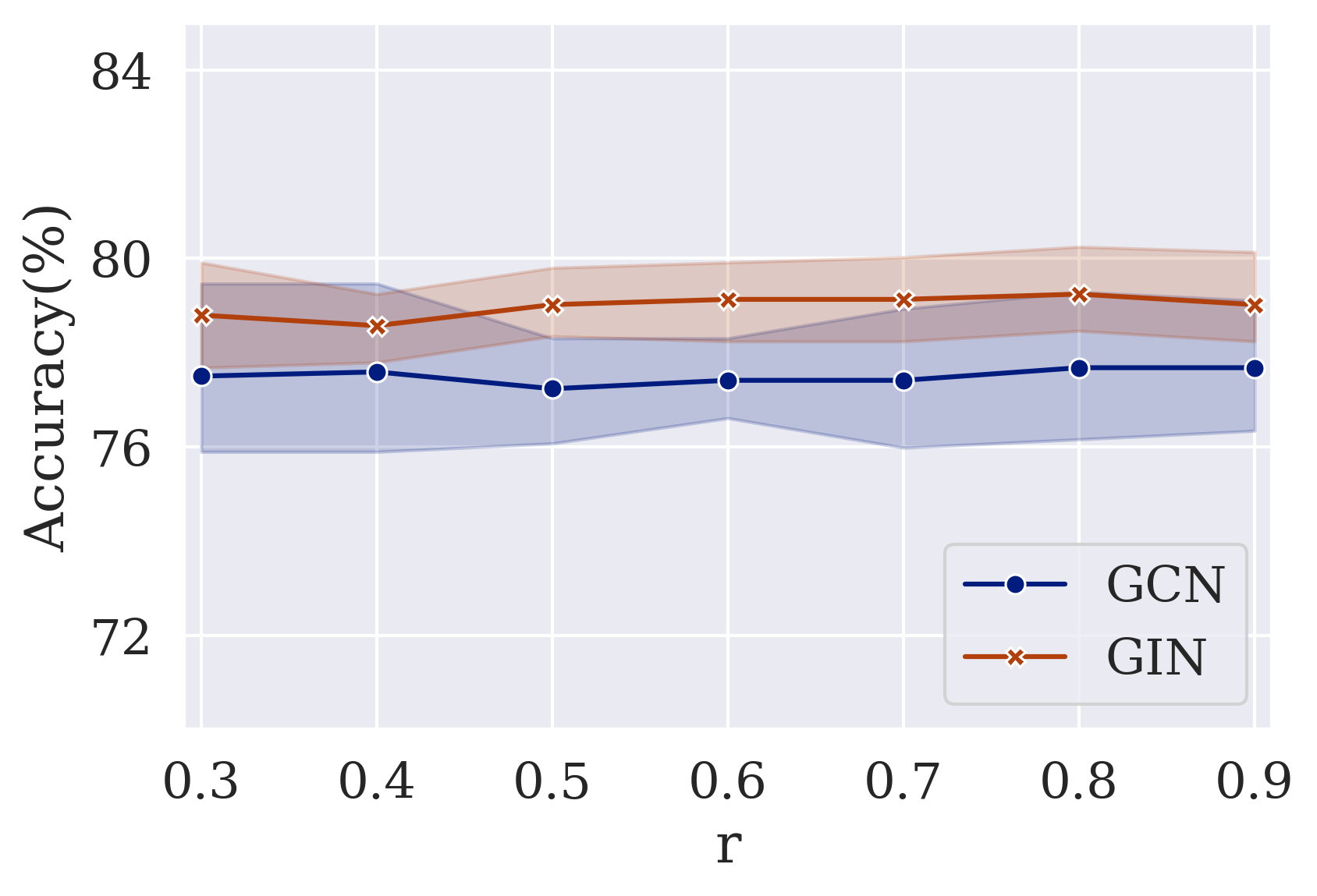}
    \end{center}
    \vspace{-7mm}
    \caption{Sensitivity analysis for the threshold information ratio $r$ on \textbf{PROTEINS}.}
    \label{fig:sensitivity}
\end{figure}
\new{As discussed in Section~\ref{sec:principle}, the number of the subspace dimension $p$ is adaptively selected for each Grassmann point, according to the global threshold of the percentage importance $r$.} The $r$ is \new{thus a new} hyperparameter \new{for} \textsc{Egg}. \new{This section designs a sensitivity analysis to demonstrate the negligible impact of} the choice of $r$ on the performance of the Grassmann embedding. \new{The model learns a graph classification task} on \textbf{Proteins} with both \textsc{GCN} and \textsc{GIN} network\new{s, which} architectures \new{and training setups are detailed in} Section~\ref{sec:expPool}. We report the mean test accuracy over $10$ \new{repetitive runs. Seven different values of the threshold ratio $r$ is set from} $0.3$ to $0.9$ with step size \new{$0.1$}. \new{The results are visualized in Figure~\ref{fig:sensitivity}, which draws a nearly horizontal trend of} the mean test accuracy \new{movement of \textsc{Egg} with different choices of the} threshold ratio $r$. This suggests that \new{the considerably wide choice of the} hyperparameter $r$ \new{does not drastically influence the} performance of \textsc{Egg} pooling. \new{In fact, we suggest} a moderately high value of $r$\new{, such as $0.5$,} to retain the essential information of a graph in Grassmann embedding \new{and guarantee a relatively fast computational speed of the algorithm at the same time}.

\subsection{Embedding Expressiveness}
\begin{figure}[t]
    \centering
    \begin{subfigure}[b]{0.32\linewidth}
        \includegraphics[scale=0.27]{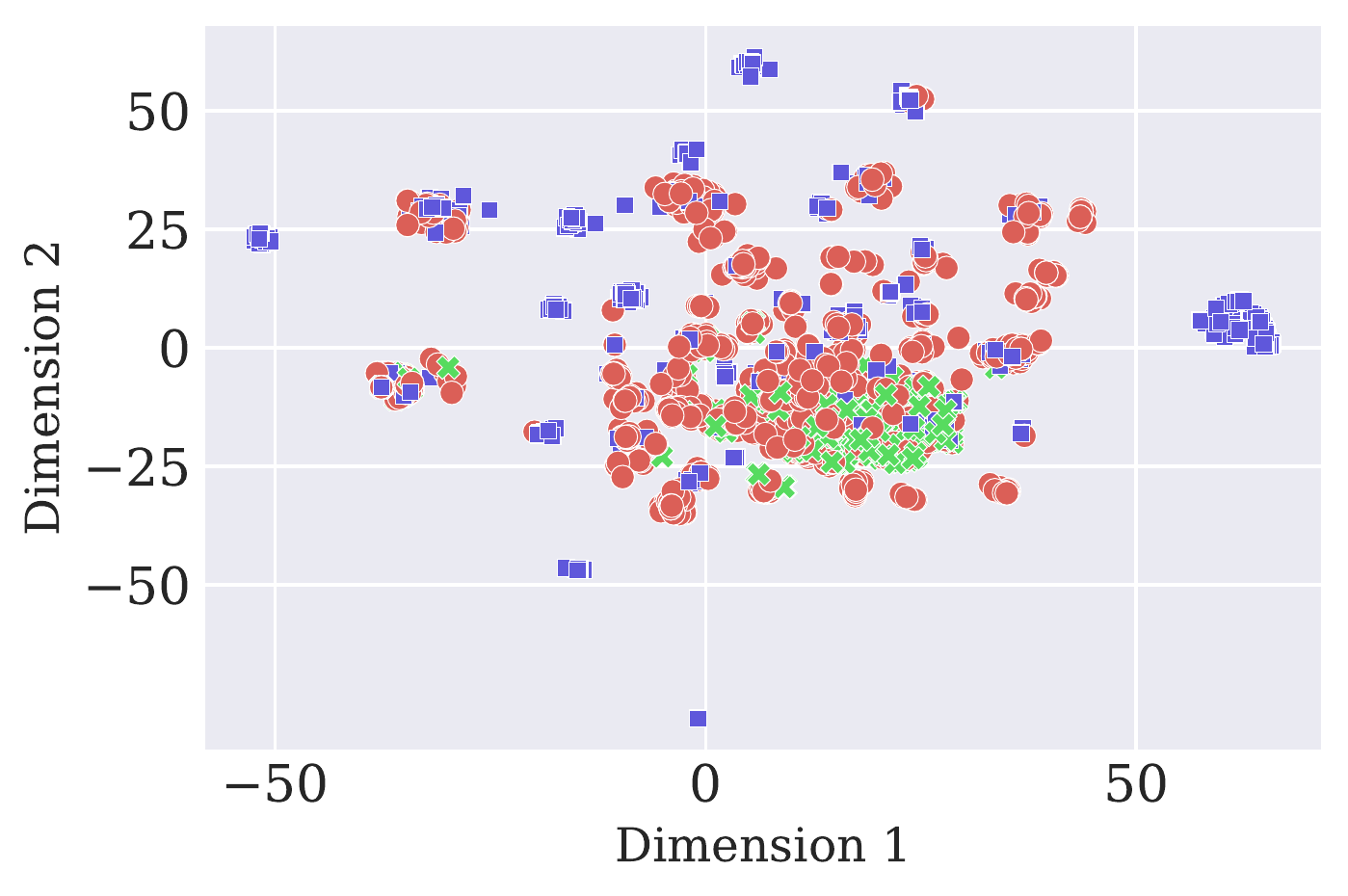}
        \caption{Raw}
        \label{fig:tsne-raw}
    \end{subfigure}
    \begin{subfigure}[b]{0.32\linewidth}
        \includegraphics[scale=0.27]{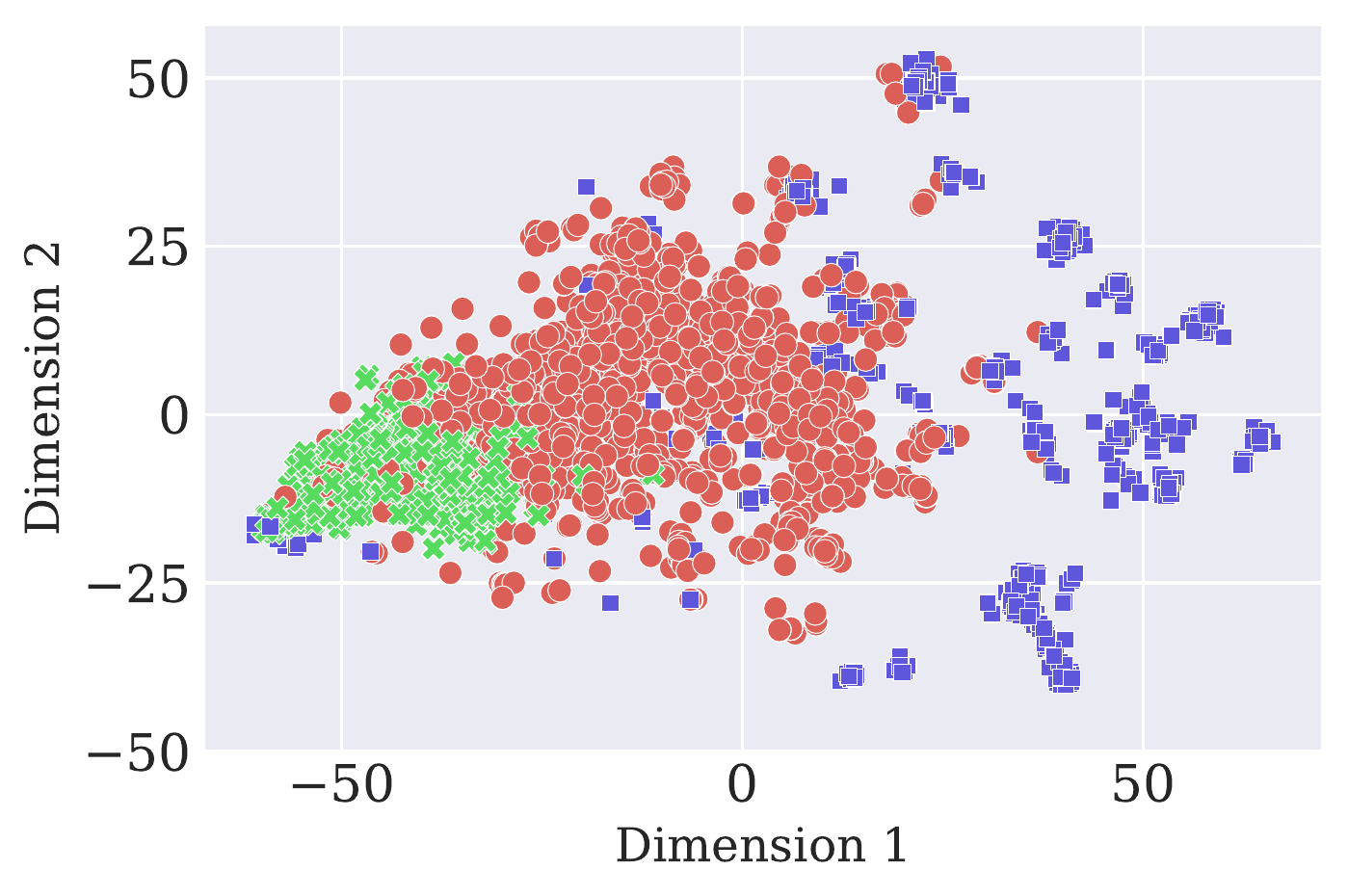}
        \caption{\textsc{GCN--Egg}}
        \label{fig:tsne-gcn}
    \end{subfigure}
    \begin{subfigure}[b]{0.32\linewidth}
        \includegraphics[scale=0.27]{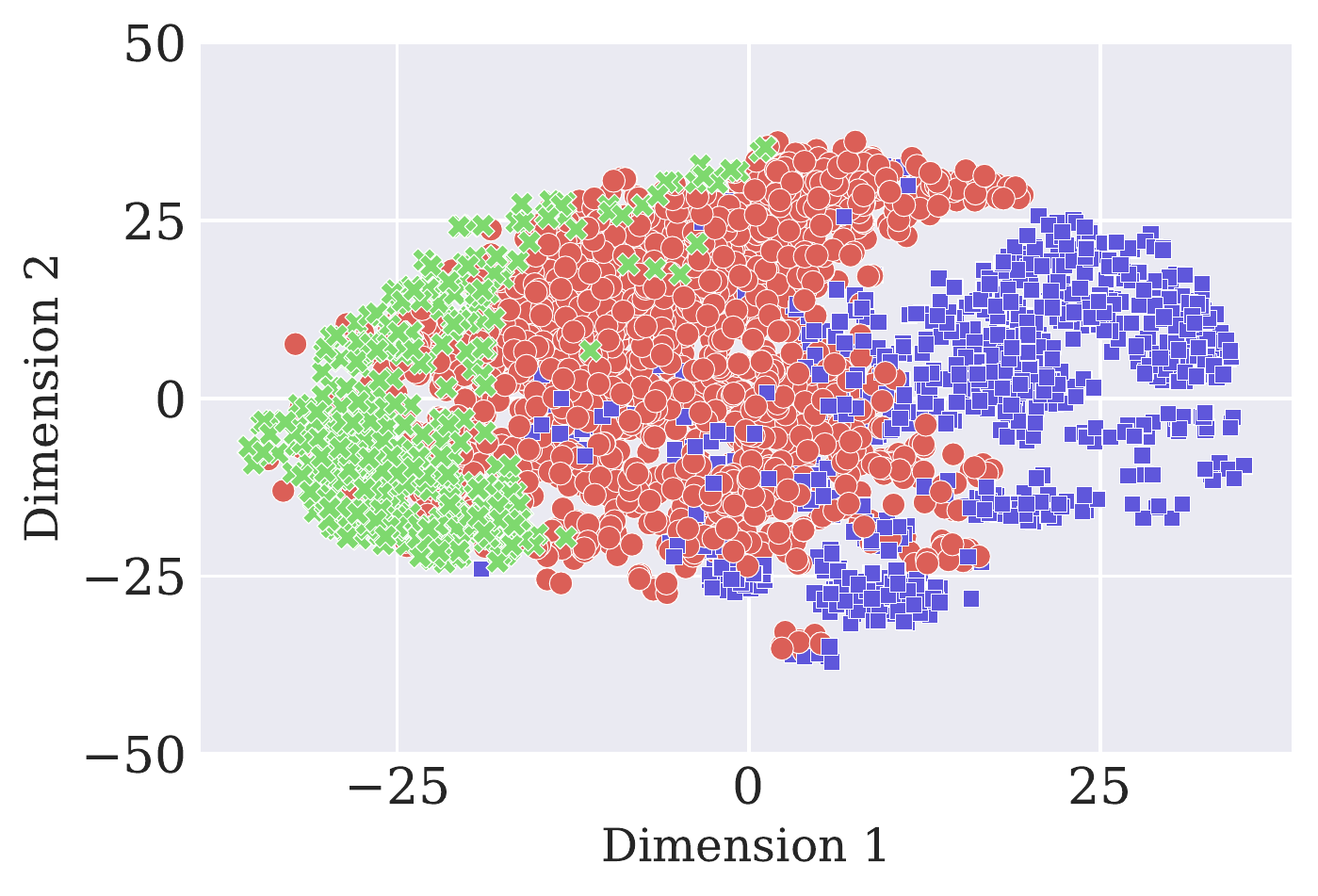}
        \caption{\textsc{GIN--Egg}}
        \label{fig:tsne-gin}
    \end{subfigure}
    \vspace{-2mm}
    \caption{The t-SNE visualizations of graph representations produced by \textsc{Egg} pooling on \textbf{COLLAB} with \textsc{GCN} before training; with \textsc{GCN} after training; and with \textsc{GIN} after training.}
\label{fig:tsne_classification}
\end{figure}
\begin{figure}[t]
    \centering
    \begin{subfigure}[b]{0.32\linewidth}
        \includegraphics[scale=0.27]{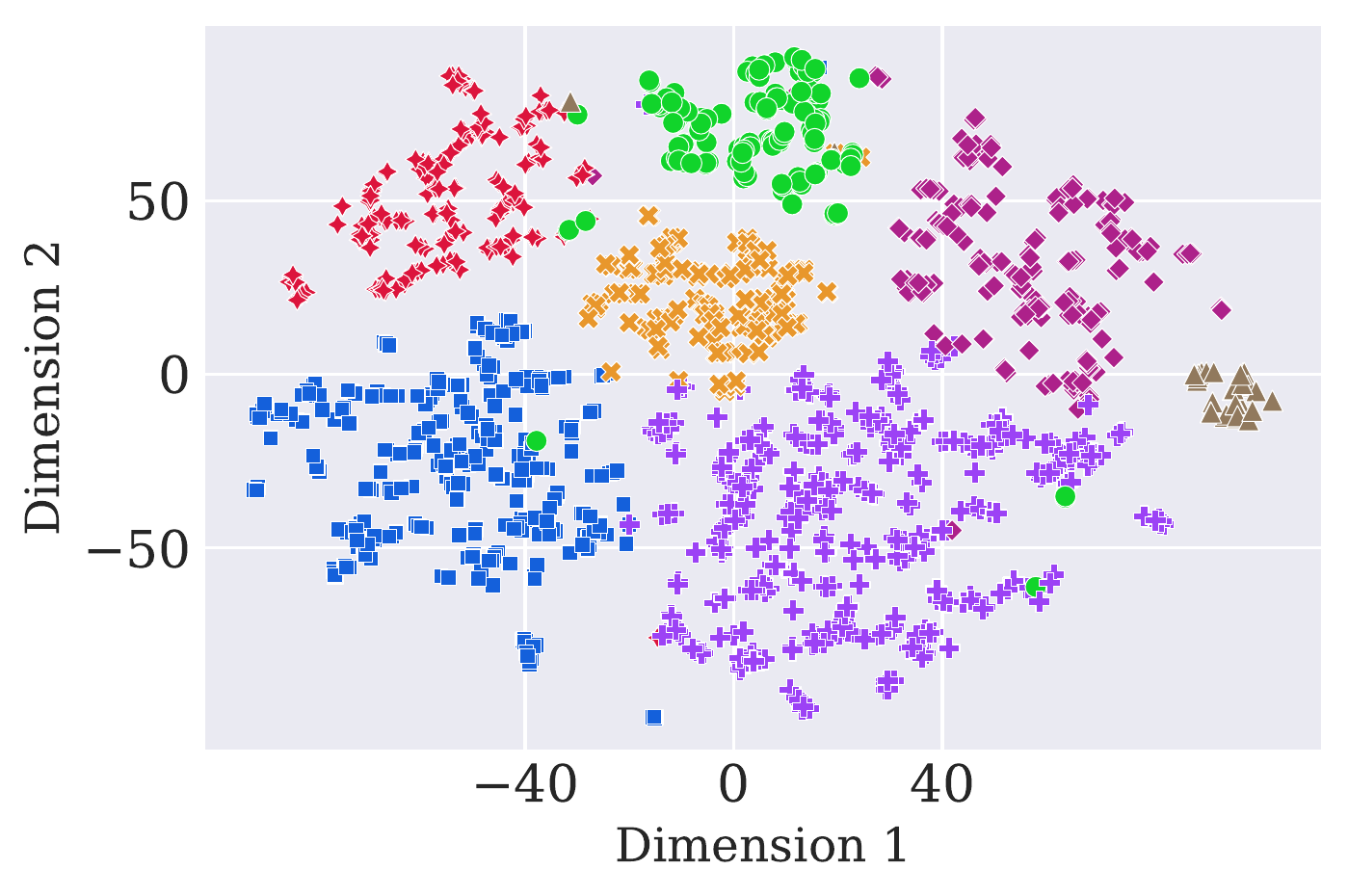}
        \caption{Cora}
    \end{subfigure}
    \begin{subfigure}[b]{0.32\linewidth}
        \includegraphics[scale=0.27]{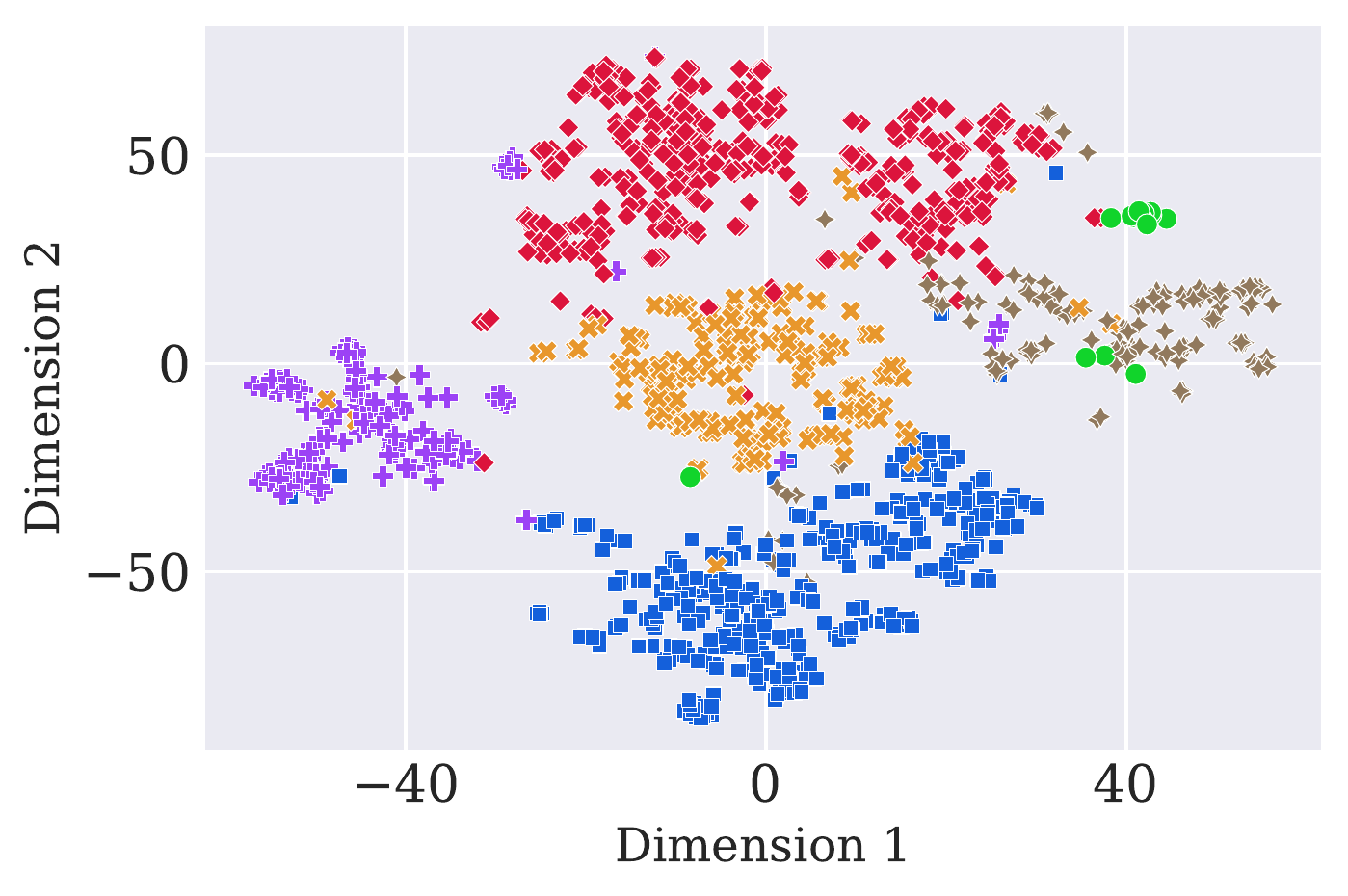}
        \caption{Citeseer}
    \end{subfigure}
    \begin{subfigure}[b]{0.32\linewidth}
        \includegraphics[scale=0.27]{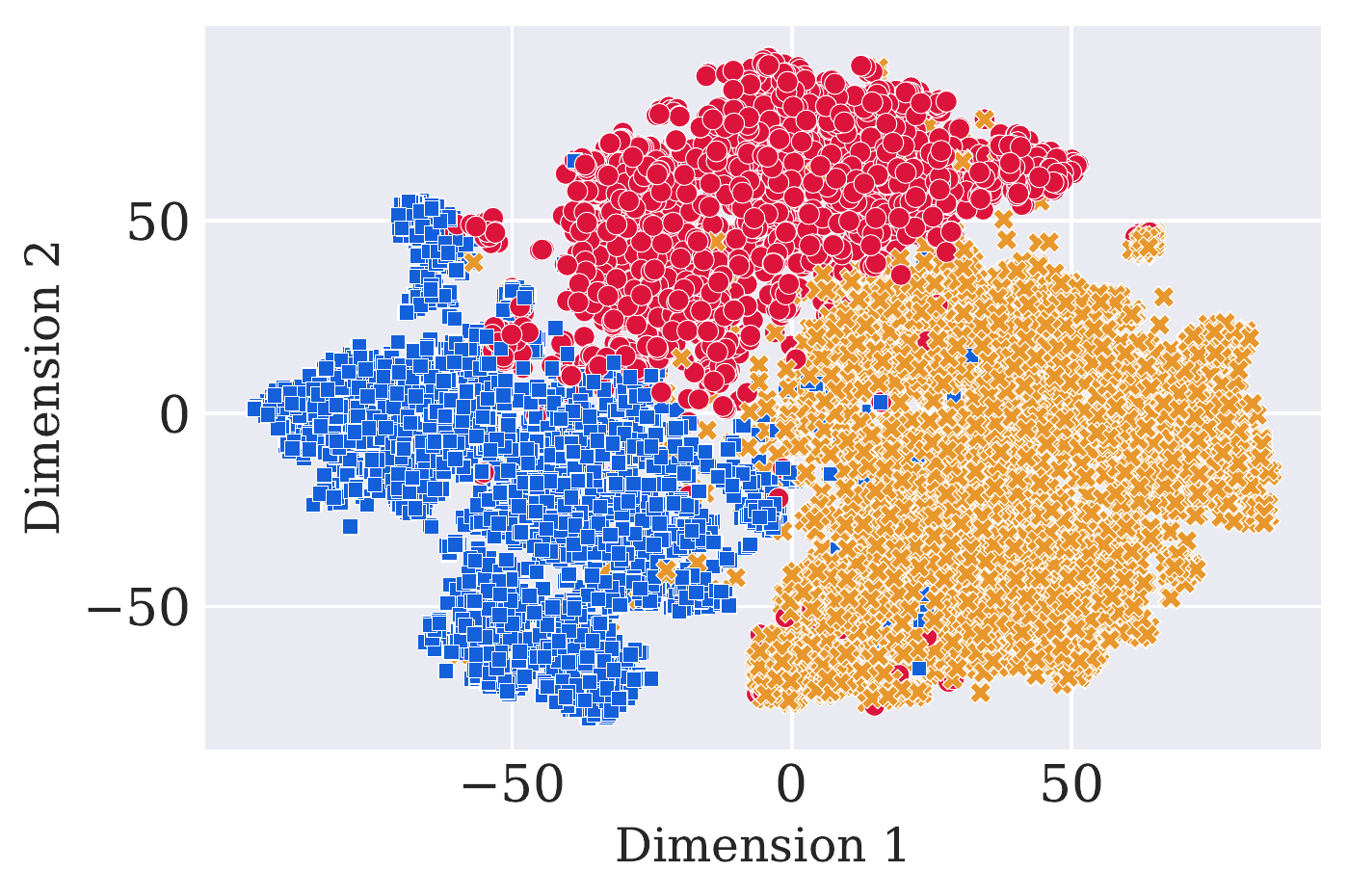}
        \caption{Pubmed}
    \end{subfigure}
    \caption{The t-SNE visualizations of node representations with \textsc{Egg}(0.8)+$k$-means on \textbf{Cora}, \textbf{Citeseer} and \textbf{Pubmed}.}
\label{fig:tsne_clustering}
\end{figure}

\new{Next, we exploit the expressiveness of the flattened Euclidean graph embeddings with} two-dimensional t-distributed Stochastic Neighbor Embedding (t-SNE). \new{The results are from the 3-class graph classification task \textbf{Collab} that we conducted in the first experiment of Section~\ref{sec:exp}. In Figure~\ref{fig:tsne_classification}, e}ach point denotes a graph \new{hidden} representation \new{by} \textsc{Egg}, and \new{the three} color\new{s} indicate \new{one of the three} true labels. \new{For a more clear presentation, we} sample $4,000$ instances \new{in random. In the case of the \textsc{GIN} convolution, outputs from all the four pooling layers are aggregated, due to the employed \textsc{JKNet} structure}. Both Figures~\ref{fig:tsne_classification}(b) and \ref{fig:tsne_classification}(c) suggest a clear clustering pattern of the pooled graphs.

\new{A similar visualization for the learned representation in node clustering tasks is displayed in Figure~\ref{fig:tsne_clustering}, where the three citation networks are trained. Once again, points of different colors are basically located at distinct corners, which implies that the hidden representations from \textsc{Egg}(0.8)+$k$-means manage to collect a differentiable pattern for the underlying clustering task. }

\subsection{Learning Behaviors}
\begin{figure}[t]
    \centering
    \begin{subfigure}[b]{0.8\linewidth}
        \includegraphics[scale=0.35]{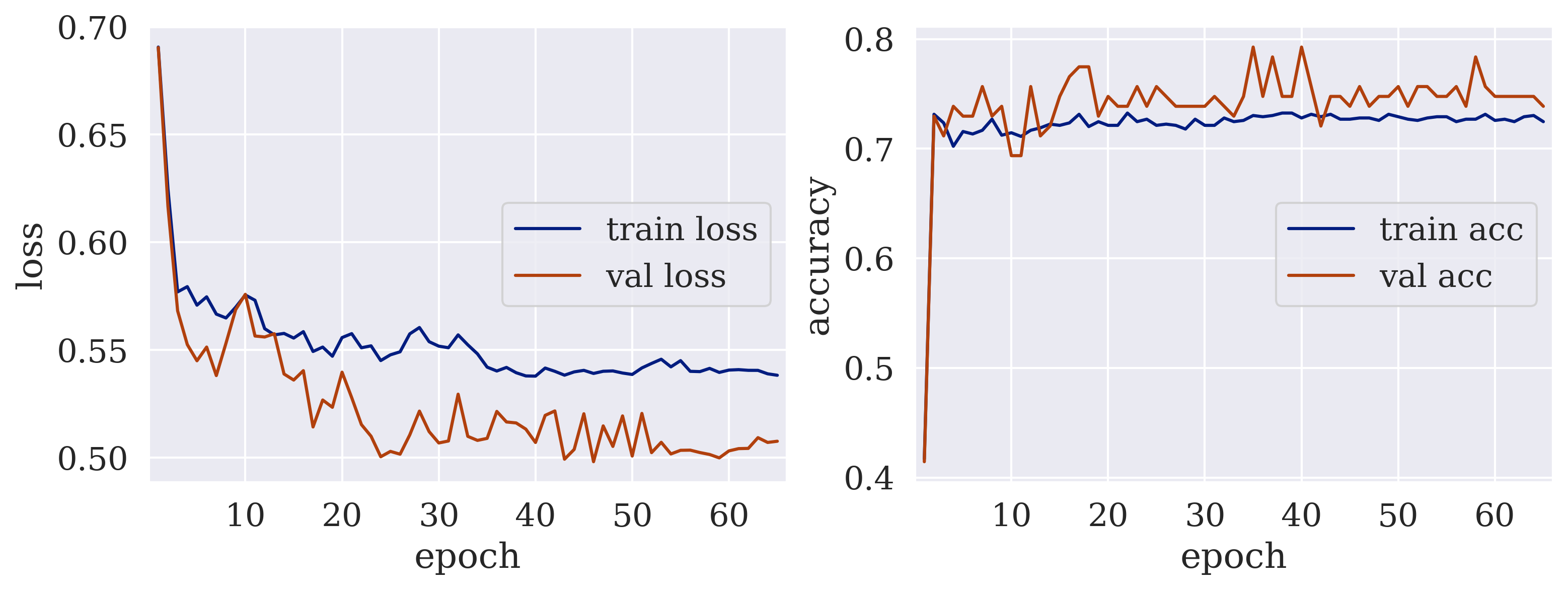}
        \caption{\textsc{GCN} Convolution}
    \end{subfigure}\\ \vspace{4mm}
    \begin{subfigure}[b]{0.8\linewidth}
        \includegraphics[scale=0.35]{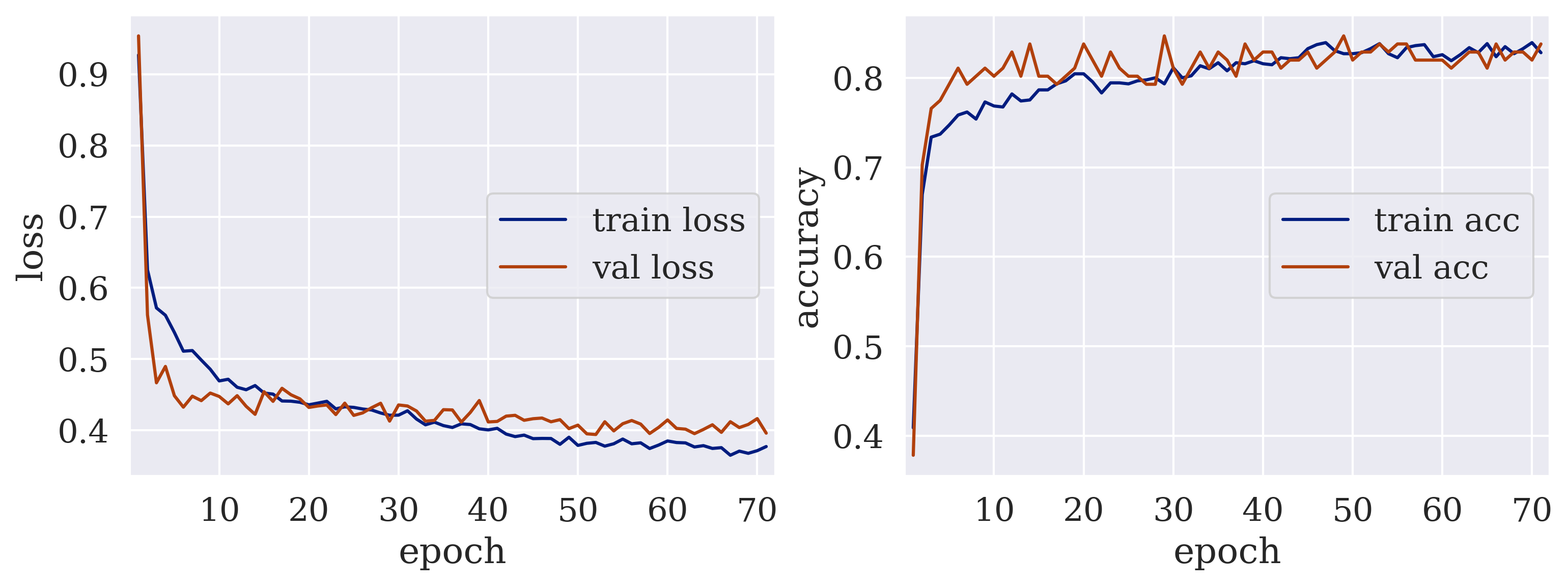}
        \caption{\textsc{GIN} Convolution}
    \end{subfigure}
    \caption{Training and Validation learning curves on the \textbf{Proteins} dataset with \textsc{Egg} pooling. The graph convolutional layers are set to \textsc{GCN} and \textsc{GIN}, respectively.}
\label{fig:lossCurve}
\end{figure}
\new{In the last investigation, we check} the training and validation curves for loss and accuracy in Figure~\ref{fig:lossCurve}. The results are retrieved from \new{the graph classification learning task on} \textbf{Proteins} \new{with the same experimental settings in} Section~\ref{sec:expPool}. \new{Here only a single run results are retrieved} due to the employment of \new{the} early stopping criteria, \new{with which every} independent run \new{could} stop at a \new{different} epoch. \new{Except for t}he minor volatility after epoch $5$ \new{that is very likely} brought about by stochastic gradients\new{, all the} four \new{training lines} validate an efficient convergence of \textsc{Egg}, where the loss curve stabilizes quickly after a few epochs. The training process is slightly longer for \textsc{GIN} convolution, which is partly due to the more sophisticated network architecture for the model to fit.

\section{Discussion and Conclusion}
\new{T}his paper \new{develops a Grassmann geometry-based} graph embedding strategy \new{named \textsc{Egg}}. \new{For a given set of hidden feature subspace of graphs, the proposed method rectifies them to Grassmann points of a Grassmann manifold to} make analysis on them. \new{Through establishing the view of treating graph nodes as a subspace, many new perspectives on formulating informative graph representations become visible. For example, this work approximates the covariance relationship of node attributes with non-linearity transformation, which concurrently offsets the inefficiency of a stack of fully-connected layers and blurs the minor perturbations of the initial representation. Furthermore, the new framework allows a swift projection from the manifold space back to the Euclidean space, so that the new representation supports common loss designs by Euclidean metrics. We demonstrate effectiveness of the embedding framework with extensive numerical experiments, for both graph-level and node-level representation learning tasks. 

The proposed \textsc{Egg} has multiple aspects of potential. As said, it} defines a non-linear \new{transformation routine to graph features, which can be exceptional when} working on complex or massive attributes. Moreover, treating regional graphs as Grassmann points \new{or other entities from non-Euclidean space brings extra flexibility to the model design. We believe this idea could} motivate more links of \new{graph topological learning and other geometric learning schemes}.

\end{document}